\documentclass[10pt, conference, letterpaper]{IEEEtran}

\IEEEoverridecommandlockouts
\usepackage[T1]{fontenc}
\usepackage{cite}
\usepackage{amsmath,amssymb,amsfonts}
\usepackage{graphicx}
\usepackage{textcomp}
\usepackage{xcolor}
\def\BibTeX{{\rm B\kern-.05em{\sc i\kern-.025em b}\kern-.08em
    T\kern-.1667em\lower.7ex\hbox{E}\kern-.125emX}}

\usepackage{microtype}
\usepackage{graphicx}
\usepackage{algorithm}
\usepackage{algpseudocode}
\usepackage{booktabs} 
\usepackage{amsthm}
\usepackage[draft]{hyperref}
\usepackage[capitalise]{cleveref}
\usepackage{footnote}

\usepackage{bm}
\usepackage{bbm}
\usepackage{graphicx}
\usepackage{thmtools}
\usepackage{thm-restate}
\usepackage{mathtools}
\usepackage[skip=0pt]{caption}
\usepackage[skip=0pt]{subcaption}
\usepackage{booktabs} 

\usepackage{threeparttable, array, float} 
\usepackage{accents}

\definecolor{red}{HTML}{E51400}  
\definecolor{blue}{HTML}{0050EF} 
\definecolor{green}{HTML}{008A00} 
\definecolor{purple}{HTML}{AA00FF} 
\definecolor{dark-red}{rgb}{0.4, 0.15, 0.15}
\definecolor{dark-blue}{rgb}{0.15, 0.15, 0.4}
\definecolor{medium-red}{rgb}{0.5, 0, 0}
\definecolor{medium-blue}{rgb}{0, 0, 0.5}
\definecolor{light-red}{rgb}{0.7, 0, 0}
\definecolor{light-blue}{rgb}{0, 0, 0.7}

\hypersetup{
    colorlinks=true,
    linkcolor=blue,
    filecolor=magenta,      
    urlcolor=blue,
}

\newtheorem{theorem}{\bf Theorem}
\newtheorem{lemma}{\bf Lemma}

\renewcommand{\thefootnote}{\fnsymbol{footnote}}
\theoremstyle{definition}
\newtheorem{remark}{\bf Remark}

\definecolor{red}{HTML}{E51400} 
\definecolor{blue}{HTML}{0050EF} 
\definecolor{green}{HTML}{008A00} 
\definecolor{purple}{HTML}{AA00FF} 
\definecolor{orange}{HTML}{FF7F00}
\definecolor{gray}{HTML}{848482}
\definecolor{Gray}{gray}{0.85}
\definecolor{LightGray}{gray}{0.96}

\usepackage{xspace}
\usepackage{comment}

\usepackage{fontawesome5}

\usepackage[colorinlistoftodos,prependcaption,textsize=tiny,textwidth=2cm,color=orange]{todonotes}

\newcommand{\carlee}[1]{{\textcolor{green}{CJW: #1}}}

\newcommand{\AlgOne}{\texttt{CUCB-SC}}
\newcommand{\AlgTwo}{\texttt{CLCB-SC-LS}}

\DeclareMathOperator*{\argmin}{argmin}
\DeclareMathOperator*{\argmax}{argmax}

\newcommand{\ubar}[1]{\underaccent{\bar}{#1}}

\newcommand{\norm}[1]{\left\lVert#1\right\rVert}

\newcommand{\abs}[1]{\left| #1 \right|}

\newcommand{\E}{\mathbb{E}}

\newcommand{\I}{\mathbb{I}}

\newcommand{\bc}{\boldsymbol{c}}

\newcommand{\bp}{\boldsymbol{p}}

\newcommand{\cA}{\mathcal{A}}
\newcommand{\cB}{\mathcal{B}}

\newcommand{\cD}{\mathcal{D}}
\newcommand{\cE}{\mathcal{E}}

\newcommand{\cG}{\mathcal{G}}

\newcommand{\cM}{\mathcal{M}}

\newcommand{\cU}{\mathcal{U}}
\newcommand{\cV}{\mathcal{V}}

\newcommand{\cT}{{\mathcal{T}}}
\newcommand{\cQ}{{\mathcal{Q}}}

\newcommand{\cX}{{\mathcal{X}}}

\newcommand{\defeq}{\vcentcolon=}

\makeatletter
\newcommand*{\rom}[1]{\expandafter\@slowromancap\romannumeral #1@}
\makeatother

\newcommand{\ts}[1]{}

\newcommand{\compilefullversion}{true}
\ifthenelse{\equal{\compilefullversion}{false}}{%
	\newcommand{\OnlyInFull}[1]{}
	\newcommand{\OnlyInShort}[1]{#1}
}{%
	\newcommand{\OnlyInFull}[1]{#1}%
	\newcommand{\OnlyInShort}[1]{}%
}%

\newcommand{\compilehidecomments}{false}
\ifthenelse{ \equal{\compilehidecomments}{true} }{%
	\newcommand{\wei}[1]{}
	\newcommand{\xutong}[1]{}
	\newcommand{\jinhang}[1]{}
	\newcommand{\siwei}[1]{}
        \newcommand{\carlee}[1]{}
        \newcommand{\xiangxiang}[1]{}

}{
\newcommand{\wei}[1]{{\color{blue}{[Wei: #1]}}}
\newcommand{\xutong}[1]{{\color{green} [Xutong: #1]}}
\newcommand{\jinhang}[1]{{\color{orange} [\text{Jinhang:} #1]}}
\newcommand{\siwei}[1]{{\color{red} [\text{Siwei:} #1]}}

\newcommand{\xiangxiang}[1]{{\color{teal} [\text{Xiangxiang:} #1]}}
}

\newcommand\blfootnote[1]{%
  \begingroup
  \renewcommand\thefootnote{}\footnote{#1}%
  \addtocounter{footnote}{-1}%
  \endgroup
}

\allowdisplaybreaks

\newenvironment{talign*}
 {\let\displaystyle\textstyle\csname align*\endcsname}
 {\endalign}



\begin{document}


\title{Semantic Caching for Low-Cost LLM Serving: \\ From Offline Learning to Online Adaptation}


\author{
\IEEEauthorblockN{Xutong Liu$^{*1}$, Baran Atalar$^{*2}$, Xiangxiang Dai$^{\dagger3}$, Jinhang Zuo$^{4}$, Siwei Wang$^{5}$,\\ John C.S. Lui$^{3}$, Wei Chen$^{5}$, Carlee Joe-Wong$^{2}$}
\IEEEauthorblockA{
$^1$University of Washington,
$^2$Carnegie Mellon University, 
$^3$The Chinese University of Hong Kong, 
\\$^4$City University of Hong Kong,
$^5$Microsoft Research
}}

\maketitle

\begin{abstract}
Large Language Models (LLMs) are revolutionizing how users interact with information systems, yet their high inference cost poses serious scalability and sustainability challenges.  Caching inference responses, allowing them to be retrieved without another forward pass through the LLM, has emerged as one possible solution. Traditional exact-match caching, however, overlooks the semantic similarity between queries, leading to unnecessary recomputation. Semantic caching addresses this by retrieving responses based on semantic similarity, but introduces a fundamentally different cache eviction problem: one must account for mismatch costs between incoming queries and cached responses. Moreover, key system parameters, such as query arrival probabilities and serving costs, are often unknown and must be learned over time. Existing semantic caching methods are largely ad-hoc, lacking theoretical foundations and unable to adapt to real-world uncertainty.
In this paper, we present a principled, learning-based framework for semantic cache eviction under unknown query and cost distributions. We formulate both offline optimization and online learning variants of the problem based on the combinatorial multi-armed bandit framework, and develop provably efficient algorithms with state-of-the-art guarantees. We also evaluate our framework on a synthetic dataset, showing that our proposed algorithms perform matching or superior performance compared with baselines. 

\end{abstract}

\section{Introduction}
The emergence of large language models (LLMs) such as GPT-4 and Gemini has enabled remarkable advances in natural language understanding, reasoning, and multimodal capabilities, powering applications ranging from coding assistants to conversational search engines. Yet this power comes at a steep cost: each query can require billions of floating point operations, making LLM inference orders of magnitude more expensive than conventional web queries~\cite{samsi2023words}. As usage continues to rise, service providers face growing concerns over latency, compute efficiency, and energy consumption.
\blfootnote{$^*$Xutong Liu and Baran Atalar are co-first authors. $^\dagger$Xiangxiang Dai is the corresponding author. The work of Carlee Joe-Wong is supported by NSF grant 2312761 and Department of Energy grant DESC0025652. The work of John C.S. Lui is supported in part by RGC GRF-14202923. Xiangxiang Dai is supported by the National Natural Science Foundation of China (625B2163).}

A natural way to mitigate these costs is through caching---reusing previously computed responses when similar queries reappear. However, most existing LLM caching frameworks are built on exact string or token matching \cite{zhu2023optimal,liu2025offline}, which are ill-suited to LLM workloads. Semantically equivalent queries such as ``What is LLM caching?'' and ``How does caching work in large language models?'' are treated as completely different entries,
 resulting in frequent cache misses even when a suitable response exists. Recent research shows that over 30\% of LLM queries are semantically similar \cite{gill2024meancache}. This exposes a major inefficiency: today's cache systems are blind to semantic similarity and conversational context. 

To address this gap, recent work has explored the problem of semantic caching, i.e., designing caching systems that reason about the meaning and context of queries, rather than their exact text
\cite{Bang2023GPTCacheAO, stogiannidis2023cache,li2024scalm, gill2024meancache}. For example, GPTCache implements an industry-scale semantic caching framework by embedding prompts into vector spaces and storing responses in a vector database for efficient similarity-based retrieval. Building on this foundation, follow-up work has incorporated student–teacher models~\cite{stogiannidis2023cache}, hierarchical clustering~\cite{li2024scalm}, and federated learning~\cite{gill2024meancache} to further improve hit rates and cost savings.

Despite these advances, there still exist limitations that can be significantly improved. First, most existing systems assume that query arrival distributions and response costs are known a priori---an assumption that rarely holds in real and uncertain deployment environments. Second, many current approaches are ad hoc and data-driven, lacking theoretical foundations and rigorous performance guarantees that are essential for understanding and optimizing semantic caching at scale.
These gaps lead to a fundamental open question:
\textit{Can we build a principled semantic caching framework that optimizes the total cost under uncertainty, with provable performance guarantees?}

\subsection{Our Contributions}
To answer this question, our contributions are as follows:

\noindent\textbf{(1) Semantic Caching Model:}
We introduce a novel semantic caching framework that generalizes traditional exact-match caching by incorporating a \emph{mismatch cost}, which models potential utility loss when serving semantically similar, but not identical, responses. This introduces a new algorithmic challenge: balancing the mismatch cost against the \emph{serving cost} incurred by querying the LLM for a fresh response. We formalize this trade-off via a unified loss function that supports \emph{arbitrary distance metrics} over queries. Based on this model, we systematically study three settings of increasing uncertainty: (i) the \emph{oracle setting}, where both query arrival probabilities and serving costs are known;
(ii) the \emph{offline learning setting}, where these parameters are unknown but learned from a static dataset; and
(iii) the \emph{online adaptive setting}, where the agent learns and adapts in real time from partial feedback.

\noindent\textbf{(2) Algorithm Design:}
For the \emph{oracle setting}, we show that computing the optimal cache is NP-hard even with full knowledge of parameters. We then prove that the loss function is non-increasing and supermodular, and design the \texttt{Reverse Greedy} algorithm with provable approximation guarantees.
For the \emph{offline learning setting}, where serving costs and query arrival probabilities are unknown but historical data is available (e.g., logs or pre-deployment feedback \cite{zhu2023optimal,liu2025offline}), we develop a pessimistic learning algorithm, \AlgOne. It estimates the parameters while penalizing uncertainty due to limited historical data, and integrates with \texttt{Reverse Greedy} to yield robust cache decisions with finite-sample guarantees.
For the \emph{online adaptive setting}, we face a unique challenge: balancing exploration and exploitation under a \emph{low-switching constraint}, since each cache update incurs additional serving cost. We propose \AlgTwo, which combines stage-based cache switching with the principle of optimism in the face of uncertainty, and is provably efficient with minimal switching.

\noindent\textbf{(3) Theoretical Guarantees:}
We provide rigorous theoretical results across all three settings grounding to combinatorial multi-armed bandits frameworks. In the oracle setting, we show our algorithm achieves near-optimal approximation with low time complexity, and can be improved to exact optimality with increased computation.
In the offline setting, we prove the suboptimality gap of our algorithm scales as $\tilde{O}(\sqrt{m/n})$, where $m$ is the number of queries, $n$ is the number of samples, and $\tilde{O}$ ignores the log factors.
In the online setting, we prove a regret bound of $\tilde{O}(\sqrt{mT})$ over $T$ rounds, with only $O(m \log \log T)$ cache switches.
Notably, our offline and online algorithms improve upon prior methods \cite{zhu2023optimal} and match state-of-the-art performance in the special case of exact-match caching \cite{liu2025offline}.

\noindent\textbf{(4) Performance Evaluation:} We conduct extensive experiments on synthetic datasets for offline and online settings. We show that the reverse greedy algorithm closely approximates the optimal cache with minimal loss. In the online setting, our algorithms consistently outperform baselines regarding average regret, with at least 11.75\% improvements observed across varying cache sizes, query distributions, and time horizons. Moreover, \AlgTwo~achieves the lowest number of cache switches and running time (reducing up to 85\%) while maintaining strong regret results.


\subsection{Related Work}

\textbf{Low-Cost LLM Serving.}
There have been tremendous efforts recently to reduce the inference cost and latency of large language model (LLM) serving, through approaches such as model quantization~\cite{dettmers2022gpt3,xiao2023smoothquant}, speculative decoding~\cite{leviathan2023fast,spector2023accelerating}, and cloud-edge offloading \cite{miao2024spotserve,fu2024serverlessllm}. A particularly promising line of work leverages caching to trade memory for cost or latency reduction. Existing studies explore caching at various levels: attention-level (KV-cache)~\cite{Pope2022EfficientlyST, Kwon2023EfficientMM, Sheng2023HighthroughputGI, Bang2023GPTCacheAO}, query-level~\cite{Gim2023PromptCM, gill2024meancache, li2024scalm, zhu2023optimal, liu2025offline}, and model/API-level~\cite{Qu2024MobileEI, Dai2024CostEffectiveOM, Feng2024GraphRouterAG}. Our work falls under the query-level LLM caching category. Unlike prior query-level semantic caching studies~\cite{Gim2023PromptCM, gill2024meancache, li2024scalm} that focus primarily on data-driven heuristics or assume known system parameters, our work provides the first unified treatment of semantic caching under uncertainty based on the combinatorial multi-armed bandits framework \cite{chen2013combinatorial,wang2017improving,liu2022batch,liu2025offline}, with provable theoretical guarantees in both offline and online learning settings.

\textbf{Offline Optimization and Online Learning for Caching.}
Classical cache eviction algorithms assume known query arrival probabilities and costs. For non-uniform arrivals, policies such as Least Frequently Used (LFU) and Least Recently Used (LRU) are commonly employed~\cite{lee2001lrfu}. When both arrival rates and costs vary, the Least Expected Cost (LEC) policy was proposed~\cite{bahn2005web}. More recent work has examined statistical learning for caching in both offline and online contexts~\cite{salem2024online,zhu2023optimal,liu2025offline}, but these efforts are largely limited to exact-match caching. We extend this line of work to semantic caching under unknown parameters, which enables us to use responses from semantically similar queries in the cache.
%
Our algorithm design and theoretical results recover prior results~\cite{zhu2023optimal} in the degenerate (exact match) case. 
Finally, while there is a line of work that studies similarity caching \cite{pandey2009nearest,chierichetti2009similarity,garetto2020similarity,salem2024online}, their formulation differs in both the loss function and algorithmic design, and does not explicitly control cache switching costs as we do.

\section{System Model}
In this section, we present the system model of the semantic caching problem for LLM serving, as shown in \cref{fig:llm_cache}.
\begin{figure}
    \centering
    \includegraphics[width=0.95\linewidth]{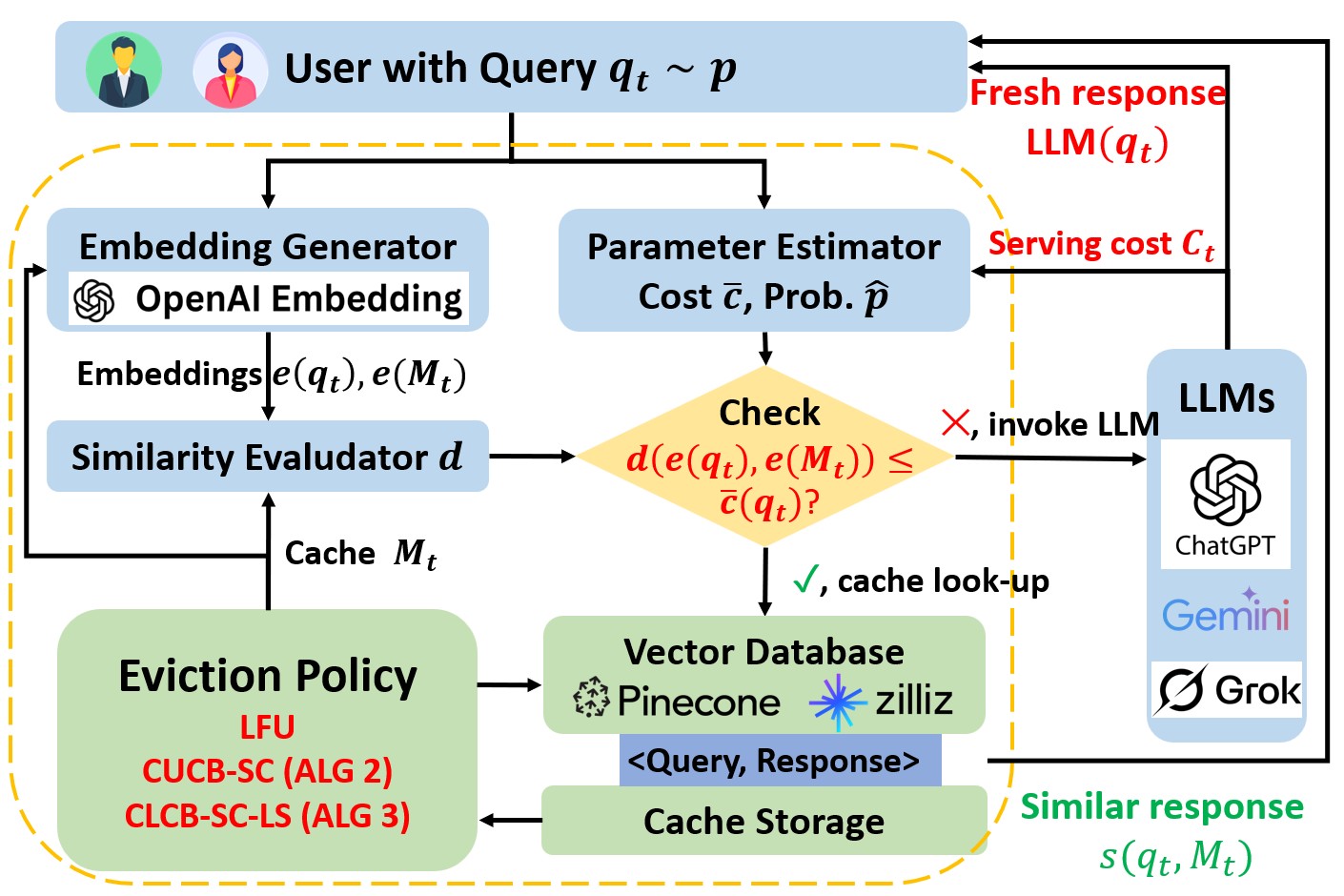}
    \caption{Semantic Caching System for Low-Cost LLM Serving
     }
    \label{fig:llm_cache}
\end{figure}

\subsection{LLM Serving System}
We consider an LLM serving system with the following core components.

\noindent\textbf{Queries and Embeddings.}
Let $\cQ = \{q_1, \ldots, q_m\}$ denote a finite set of $m \in \mathbb{Z}_+$ distinct queries, where each query $q_i \in \cX$ is a natural language prompt (e.g., ``What is LLM caching?'') drawn from a language space $\cX$.
The system operator (i.e., learning agent) has access to a large language model $(\texttt{LLM})$ (e.g., ChatGPT) and an associated embedding generator $e: \cX \rightarrow \mathbb{R}^{d_e}$ (e.g., OpenAI Embedding API), which maps each query into a $d_e$-dimensional vector representation.

\noindent\textbf{Arrival Probabilities and Serving Costs.}
Each query $q \in \cQ$ is associated with an unknown expected serving cost $c(q) \in (0, 1]$ and an unknown query arrival probability $p(q) \in (0, 1]$, such that $\sum_{q \in \cQ} p(q) = 1$. For example, this cost could represent the latency or the computational resources expended from a forward pass through the LLM.
At each round $t$, whenever the agent invokes the \texttt{LLM} to serve a query $q$, it incurs a random cost $C_t(q) = c(q) + \epsilon_t(q)$, where $\epsilon_t(q)$ is zero-mean sub-Gaussian noise capturing cost variability (which is due to fluctuating latency or random output token length) \cite{zhu2023optimal}.

\noindent\textbf{Similarity Distance and Mismatch Costs.}
We define a similarity distance function $d: \cQ \times \cQ \rightarrow \mathbb{R}_+$, where a smaller value indicates stronger semantic similarity (e.g., the distance between ``What is LLM caching?'' and ``How does caching work in large language models?'' is small).
We adopt the Euclidean distance as the default similarity metric: $d(q_1, q_2) = \|e(q_1) - e(q_2)\|_2$, yet we can generalize to \textit{any metric}, such as cosine distance $d(q_1, q_2)=e(q_1)^{\top}e(q_2)/(\norm{e(q_1)}_2 \norm{e(q_2)}_2)$, without changing the results of the current work. 

We further define the distance from a query $q$ to a set $\cM \subseteq \cQ$ as $d(q, \cM) \coloneqq \min_{u \in \cM} d(q, u)$, and let $s(q, \cM) \coloneqq \argmin_{u \in \cM} d(q, u)$ denote the closest cached query to $q$.

When a query $q$ is served, the agent may either invoke the \texttt{LLM} to generate a fresh response $a(q)$, or reuse a cached response $a(u)$ from a similar query $u \in \cM$, where $\cM$ represents the set of cached queries and their responses.
Reusing a cached response incurs an expected \emph{mismatch cost} $\gamma \cdot d(q, u)$, where $\gamma > 0$ is a scaling parameter. Since our results hold for any distance function $d$, we normalize this cost such that (i) the fresh response of $q$ incurs zero expected mismatch cost, and (ii) we set $\gamma = 1$ without loss of generality. 

\subsection{Semantic Caching for Low-cost LLM Serving}\label{sec:SC_serving}
To reduce inference cost, the system maintains a semantic cache $\cM$ containing at most $k$ query-response pairs:
\[
\cM = \{(q_1, a(q_1)), \ldots, (q_k, a(q_k))\},
\]
stored in a vector database \cite{pan2024survey}.  
Let $\bp = (p(q_1), \ldots, p(q_m))$ and $\bc = (c(q_1), \ldots, c(q_m))$ denote the arrival probability and serving cost vectors, respectively.

In each round $t$, a query $q_t \sim \bp$ arrives. The agent can decide between two actions:
\begin{itemize}
    \item \textbf{Cache Lookup:} Return a cached response $a(q)$ for some $q \in \cM$, incurring a mismatch cost $d(q_t, q)$;
    \item \textbf{LLM Query:} Invoke the LLM to generate a fresh response $a(q_t)$, incurring the expected cost $c(q_t)$.
\end{itemize}

Obviously, the optimal decision rule is to compare the cost of querying the LLM $c(q)$ with the mismatch cost $d(q_t, \cM)$ incurred by using the most similar cached query $s(q_t, \cM)= \argmin_{u\in \cM} d(q_t,u)$: 
\begin{align}
    a_t = 
    \begin{cases}
        \texttt{LLM}(q_t) & \text{if } c(q_t) \le d(q_t, \cM), \\
        a(s(q_t, \cM)) & \text{otherwise}.
    \end{cases}
    \label{eq:rule}
\end{align}

This rule then raises the question: \textit{how should we determine the query-response pairs stored in the cache?} In the remainder of the paper, we seek to answer this question.

Given the random query arrivals, the expected loss incurred by a cache $\cM$ is defined as:
\begin{align}\label{eq:loss}\textstyle
    \ell(\cM; \bp, \bc, d) = \sum_{q \in \cQ} p(q) \cdot \min\{c(q), d(q, \cM)\}
\end{align}
when using the optimal decision rule in \eqref{eq:rule}.
Our goal is to find optimal $\cM^*$ of size at most $k$ that minimizes this loss:
\begin{align}\label{eq:opt_prob}\textstyle
    \cM^* = \argmin_{\cM \subseteq \cQ, |\cM| \le k} \ell(\cM; \bp, \bc, d).
\end{align}

\begin{remark}[Special Cases of the distance function $d$]\label{rmk:special}
Our framework captures a wide range of semantic and structural caching models by varying the distance function $d$:

\textbf{(i) Bipartite Graph Coverage.}
Consider a threshold-based distance function for some $\epsilon > 0$:
\begin{align}\label{eq:distance_bipartite}
d(q_1, q_2) = 
\begin{cases}
0 & \text{if } \|e(q_1) - e(q_2)\|_2 \le \epsilon, \\
1 & \text{otherwise}.
\end{cases}
\end{align}
This induces a bipartite graph $\cG = (\cU, \cV, \cE)$ where $\cU = \cV = \cQ$ and $(u, v) \in \cE$ if and only if $d(u, v) = 0$. Define $N(\cM) = \{v \in \cV : \exists\, u \in \cM \text{ s.t. } (u,v) \in \cE\}$ as the set of covered nodes. Then the general loss as in \cref{eq:loss} becomes:
\begin{equation}\label{eq:bipartite_loss}\textstyle
    \ell(\cM; \bp, \bc, d) = \sum_{v \notin N(\cM)} p(v) c(v).
\end{equation}
This special case is used in proving NP-hardness (\Cref{lem:NP}).

\textbf{(ii) Exact-Match Cache.}
Setting $\epsilon = 0$ in the threshold distance of \cref{eq:distance_bipartite} reduces the model to exact-match caching as in prior work \cite{zhu2023optimal, liu2025offline}, where only identical queries match. The loss simplifies to
$\ell(\cM; \bp, \bc, d) = \sum_{q \notin \cM} p(q) c(q).$
\end{remark}

In later sections, we conduct a systematic study on how to find $\cM^*$ under different settings of increasing uncertainty:
\begin{enumerate}
     \item \textbf{Oracle setting (\cref{sec:approx}):} $\bp$ and $\bc$ are known a priori;
    \item \textbf{Offline learning setting (\cref{sec:offline}):} an offline dataset is used to estimate the unknown $\bp$ and $\bc$;
    \item \textbf{Online learning setting (\cref{sec:online}):} the agent can adaptively change the cache in each round, and the unknown $\bp,\bc$ are learned through sequential user interactions.
\end{enumerate}

\section{Approximate Solution for Semantic Caching with Known Parameters}\label{sec:approx}
In this section, we study the semantic caching problem when the system parameters are known a priori.
We first establish the computational hardness of the problem \cref{eq:opt_prob}, even for a special case of the general loss:
\begin{lemma}[NP-hardness]\label{lem:NP}
It is NP-hard to compute the optimal cache $\cM^*$ defined in \eqref{eq:bipartite_loss}.
\end{lemma}

Given this computational hardness, a natural question is whether a good approximate solution can be obtained efficiently. We answer this in the affirmative by leveraging the non-increasing and supermodular property of the loss function, which enables the use of greedy algorithms with provable guarantees. Intuitively, non-increasing means that as we cache more query-answer pairs, the expected loss decreases; supermodularity further implies that the marginal benefit of caching an additional item diminishes as the cache grows.

\begin{lemma}[Loss Function Properties]\label{lem:supermodular}
The loss function $\ell(\cM; \bp, \bc, d)$ is non-increasing 
and supermodular regarding the set $\cM$. That is, for any $\cA \subseteq \cB \subseteq \cQ$ and $q \in \cQ \setminus \cB$, we have:
    (i) $\ell(\cB; \bp, \bc, d) \le \ell(\cA; \bp, \bc, d)$, and (ii)
    $\ell(\cA \cup \{q\}; \bp, \bc, d) - \ell(\cA; \bp, \bc, d) \le \ell(\cB \cup \{q\}; \bp, \bc, d) - \ell(\cB; \bp, \bc, d)$.
\end{lemma}

Based on this supermodularity property, we design a \emph{reverse greedy} algorithm in \cref{alg:rg}. Unlike standard greedy algorithms that add items from scratch, \cref{alg:rg} starts by assuming all query-answer pairs are cached and then iteratively removes the least useful cached item until only $k$ remain. 

\begin{algorithm}[!t]
    \caption{\texttt{Reverse\_Greedy} \cite{il2001approximation}: Removing Least Beneficial Cache Items 
    }
    \label{alg:rg}
    \begin{algorithmic}[1]
        \State \textbf{Input:} Queries $\cQ$, probabilities $\bp$, serving costs $\bc$, distance function $d$, cache size $k$
        \State \textbf{Initialize:} $\cM_0 \gets \cQ$
        \For{$i = 1$ to $m-k$}
            \State Compute $\ell(\cM_{i-1} \setminus \{q\}; \bp, \bc, d)$ for all $q \in \cM_{i-1}$\label{line:oracle_l}
            \State Select $q_i = \argmin_{q \in \cM_{i-1}} \ell(\cM_{i-1} \setminus \{q\}; \bp, \bc, d)$\label{line:oracle_min}
            \State Update $\cM_i \gets \cM_{i-1} \setminus \{q_i\}$
        \EndFor
        \State \textbf{Return:} $\cM = \cM_{m-k}$
    \end{algorithmic}
\end{algorithm}

We now provide a theoretical guarantee on the performance of Alg. \ref{alg:rg} following \cite{il2001approximation}. 
Let us define the curvature $c=1 - \min_{q \in \cQ} \frac{\ell(\cQ \backslash \{q\} ; \bp, \bc, d) - \ell(\cQ; \bp, \bc, d) }{\ell(\emptyset;\bp, \bc, d) - \ell(\{q\}; \bp, \bc, d)}$, measuring how much the function deviates from the additive function (whose $c$ is 0). 

\begin{theorem}[Approximation Guarantee \cite{il2001approximation}]\label{thm:approx}
Let $\cM$ be the solution returned by \Cref{alg:rg}. Then,
$\ell(\cM; \bp, \bc, d) \le \frac{e^\beta - 1}{\beta} \cdot \ell(\cM^*; \bp, \bc, d)$,
where $\beta = \frac{c}{1 - c}$ and $c \in [0, 1]$ is the curvature.
\end{theorem}


\section{Offline Learning for the Semantic Caching Problem with Unknown Parameters}\label{sec:offline}

\begin{algorithm}[!t]
    \caption{\texttt{CUCB-SC}: Combinatorial Upper Confidence Bound Algorithm for the Semantic Caching Problem}
    \label{alg:CUCB-SC}
    \begin{algorithmic}[1]
        \State \textbf{Input:} Dataset $\cD=\left\{ \left( \cM_t, q_t, C_t \right) \right\}_{t=1}^{n}$, queries $\cQ$, solver \texttt{Reverse\_Greedy}, embedding generator $e$, distance function $d$, probability $\delta$.
		\For{query $q \in \cQ$}
            \State Calculate counters $N(q)=\sum_{t=1}^n \I\{q = q_t\}$ and $N_c(q)=\sum_{t=1}^n \I\{q= q_t \text{ and } C_t \neq \emptyset\}$.\label{line:LLM_counter}
            \State Calculate empirical means $\hat{p}(q)=N(q)/n$ and $\,\hat{c}(q){=}\sum_{t \in [n]} \I\{q= q_t \text{ and } C_t \neq \emptyset\}C_t/N_c(q)$.\label{line:LLM_empirical_mean}
            \State Calculate UCB of the cost $\bar{c}(q)=\hat{c}(q)+\sqrt{ \frac{\log (\frac{6mn}{\delta})} {2 N_c(q)} }$.\label{line:LLM_UCB}
        \EndFor
        \State Compute $\hat{\cA}= \texttt{Reverse\_Greedy} \left(\cQ, \hat{\bp}, \bar{\bc}, d, k\right).$\label{line:LLM_oracle}
        \State Call \texttt{LLM} to obtain responses $a(q)$ for $q \in \hat{\cA}$.
        \State \textbf{Return:} $\hat{\cM}=\{q, a(q)\}_{q \in \hat{\cA}}$.\label{line:offline_call_oracle}
	\end{algorithmic}
\end{algorithm}

In this section, we consider the setting where the cost and arrival probabilities are unknown and must be learned from offline data. We first introduce the formal offline learning setup, then present the \AlgOne~algorithm in \Cref{alg:CUCB-SC} for the semantic caching problem, and finally provide its theoretical analysis and implications.

\subsection{Offline Learning Setup}
We assume access to a historical dataset $\cD = {(\cM_t, q_t, C_t)}_{t=1}^n$ collected by an experimenter, such as logs of user interactions with an LLM service.
where $\cM_t$ is the selected cache at round $t$, $q_t$ is the user query, and $C_t = C_t(q_t)$ is the observed cost. If $C_t = \emptyset$, it indicates that the cached response from the nearest neighbor in $\cM_t$ was used and no cost was observed. If $C_t > 0$, the LLM was invoked and the cost feedback was recorded. Define $\nu(q) = \Pr[C_t(q_t) \ne \emptyset \mid q_t = q]>0$ as the probability of observing the cost feedback for query $q$.

Since computing the optimal cache is NP-hard even with known parameters (\Cref{lem:NP}), our goal is to learn an approximate solution that competes with the $\alpha$-approximate oracle defined given by \cref{alg:rg}. Specifically, we aim to minimize the approximate suboptimality:
\begin{align}\label{eq:subopt_def}
\text{SubOpt}(\cM) = \ell(\cM;\bp,\bc,d) - \alpha \cdot \ell(\cM^*; \bp, \bc, d),
\end{align}
where $\alpha$ is the approximation ratio from \Cref{thm:approx}.


\subsection{Algorithm Design of \AlgOne}
The \AlgOne~algorithm (shown in \Cref{alg:CUCB-SC}) proceeds in two key steps.
First, it computes high probability estimates of the arrival probabilities $\bp$ and the serving costs $\bc$. 
For the serving costs, we use the upper confidence bounds (UCB) as pessimistic estimations of the true query costs.
Specifically, we penalize each cost estimation $\hat{c}(q)$ by its confidence interval, $\sqrt{ \log (\frac{6mn}{\delta})/ {2N_c(q)} }$ in \cref{line:LLM_UCB}.
This approach, rooted in the pessimism principle \cite{jin2020provably}, mitigates the impact of high fluctuations in empirical estimates caused by limited observations, effectively addressing the uncertainty inherent in passively collected data.
For the arrival probabilities, we leverage the full-feedback nature of the dataset and use empirical means $\hat{p}(q)$. Secondly, we invoke the reverse greedy algorithm, which iteratively removes the least beneficial query-response pair based on estimated $\hat{\bp}$ and $\bar{\bc}$, to select the cache $\hat{\cM}$ in \cref{line:offline_call_oracle}. 

\subsection{Theoretical Result and Its Discussion}
We now present the performance guarantee for \AlgOne.

\begin{theorem}[Suboptimality Gap Bound for \cref{alg:CUCB-SC}]\label{thm:LLM_cache_main_offline}
Let $\cD$ be an offline dataset of $n$ samples. Suppose $n \ge \frac{8 \log (1/\delta)}{\min_{q \in \cQ} p(q)\nu(q)}$. Then with probability of at least $1 - \delta$, the cache $\hat{\cM}$ returned by \AlgOne~satisfies:
$\text{SubOpt}(\hat{\cM})
\le 4 \sqrt{\frac{2 \sum_{q \in \cQ} \frac{1}{\nu(q)} \log(\frac{6mn}{\delta})}{n}}.$
\end{theorem}

\begin{remark}[Discussion of \cref{thm:LLM_cache_main_offline}]
    According to \cref{thm:LLM_cache_main_offline}, the suboptimality gap decreases at a rate of $\sqrt{1/n}$ with respect to the number of offline data samples $n$. In addition, the gap scales proportionally with $\sqrt{\sum_{q \in \cQ} 1/\nu(q)}$, where $\nu(q)$ can be interpreted as the data generation rate for query $q$. 
    In the special case where $\nu(q) = 1$ for all $q$ meanwhile the cache operates in an exact-match setting (see \cref{rmk:special}.(ii)), our formulation reduces to the setting in Zhu et al.~\cite{zhu2023optimal}. Under this regime, our result improves upon theirs by a factor of $O(k/\sqrt{C_1})$, where $C_1 = \min_{q \in \cQ} c(q)$ is the minimal query cost. This improvement arises from our separate treatment of arrival and cost distributions, leveraging their distinct feedback structures, in contrast to the coupled treatment in Zhu et al.~\cite{zhu2023optimal}.
    More recently, Liu et al.~\cite{liu2025offline} achieve a similar improvement while also operating in the degenerate exact-match cache setting. Our work not only matches their theoretical guarantees but also extends them to the more general semantic caching setting.
    	
\end{remark}

\subsection{Suboptimal Gap Analysis} 
To establish the result in \Cref{thm:LLM_cache_main_offline}, we begin with the following concentration bounds.

\begin{lemma}[Concentration of Estimates \cite{liu2025offline}]\label{lem:offline_concen}
For any $\delta > 0$, define the following events:
$\cE_{\text{arvl}} \coloneqq \Big\{ \sum_{q \in \cQ}|\hat{p}(q) - p(q)| \le \sqrt{ \frac{2m \log(2/\delta)}{n} } \Big\}, \quad
\cE_{\text{cost}} \coloneqq \Big\{ |\hat{c}(q) - c(q)| \le \sqrt{ \frac{\log(2mn/\delta)}{2N_c(q)} }, \; \forall q \in \cQ \Big\}, \quad
\cE_{\text{counter}} \coloneqq \Big\{ N_c(q) \ge \frac{n \cdot p(q) \nu(q)}{2}, \; \forall q \in \cQ \Big\}.$
Then, assuming $n \ge \frac{8 \log (m/\delta)}{\min_{q} p(q) \nu(q)}$, all three events occur with probability at least $1 - \delta$.
\end{lemma}

Given arm-level estimation errors in \cref{lem:offline_concen}, we next proceed to relate the arm-level error to the cache-level suboptimal gap. Specifically, we bound the sensitivity of the loss function to perturbations in its parameters.

\begin{lemma}[Loss Function Sensitivity]\label{lem:offline_smooth}
For any cost parameters $\bc, \bc' \in [0,1]^m$, arrival probabilities $\bp, \bp' \in [0,1]^m$, and any distance function $d: \cQ \times \cQ \rightarrow [0,1]$, the loss function satisfies that for any cache $\cM \subseteq \cQ$, $\abs{\ell(\cM; \bc, \bp, d) - \ell(\cM; \bc', \bp', d)} \le \sum_{q \in Q}|p(q)-p'(q)| + \sum_{q \in Q} p(q) |c(q) - c(q')|$.
\end{lemma}

\begin{proof}[Proof of \cref{thm:LLM_cache_main_offline}]
We ignore $d$ in the loss function $\ell(\cM;\bp,\bc,d)=\ell(\cM;\bp,\bc)$ when contexts are clear.
Under the events of $\cE_{\text{arvl}}, \cE_{\text{cost}}, \cE_{\text{counter}}$,
$\ell(\hat{\cM};\bc, \bp) - \alpha \ell(\cM^*;\bc, \bp) 
= \underbrace{\alpha \ell(\cM^*; \bar{\bc}, \hat{\bp}) - \alpha \ell(\cM^*; \bc, \bp)}_{\text{estimation gap}} 
+ \underbrace{\ell(\hat{\cM}; \bar{\bc}, \hat{\bp}) - \alpha \ell(\cM^*;\bar{\bc}, \hat{\bp})}_{\text{greedy gap}} 
\\+ \underbrace{\ell(\hat{\cM};\bc, \bp) - \ell(\hat{\cM};\bar{\bc}, \hat{\bp})}_{\text{pessimism gap}} 
\le \alpha \ell(\cM^*; \bar{\bc}, \hat{\bp}) - \alpha \ell(\cM^*; \bc, \bp) + \ell(\hat{\cM}; \bar{\bc}, \bp) - \ell(\hat{\cM};\bar{\bc}, \hat{\bp}) 
\le \sum_{q \in \cQ} p(q) |\bar{c}(q) - c(q)| + 2\norm{\hat{\bp} - \bp}_1 
\le \sum_{q \in \cQ} p(q) |\bar{c}(q) - c(q)| + 2\sqrt{\frac{2m\log(\frac{2}{\delta})}{n}},$
where the equality is due to adding and subtracting $\alpha \ell(\cM^*; \bar{\bc}, \hat{\bp})$ and $\ell (\hat{\cM};\bar{\bc}, \hat{\bp})$, the first inequality is due to the greedy gap $\le 0$ by the approximation ratio guarantee of the $\texttt{Reverse Greedy}$, $\bar{\bc} \ge \bc$ by \cref{lem:offline_concen}, and $\ell(\cM;\bp,\bc)$ is monotone regarding $\bc$, the second inequality is due to \cref{lem:offline_smooth}, and the last inequality is due to \cref{lem:offline_concen}.

Then we have for the first term: 
$\sum_{q \in \cQ} p(q) \left| \bar{c}(q) - c(q) \right| 
\le 2\sum_{q \in \cQ} p(q) \sqrt{ \frac {\log (2mn/\delta) } {2 N_c(q)} } 
\le 2\sum_{q \in \cQ} p(q) \sqrt{ \frac {\log (2mn/\delta) } {n \cdot p(q) \nu(q)} } 
\le 2 \sqrt{\sum_{q \in \cQ} p(q)} \sqrt{\sum_{q \in \cQ} \frac {p(q) \log (2mn/\delta) } {n \cdot p(q) \nu(q)} } = 2 \sqrt{ \frac {\sum_{q \in \cQ} \frac{1}{\nu(q)} \log (2mn/\delta) } {n} }$,
where the first two inequalities are due to the definition of high-probability events in \cref{lem:offline_concen}, the third inequality is due to the Cauchy-Schwarz inequality. Finally, we can set $\delta'=1/(3\delta)$ to conclude the theorem.
\end{proof}


    \section{Online Adaptive Learning for the Semantic Caching Problem with Unknown Parameters}\label{sec:online}
Now that we have developed a caching policy for the offline setting, we move to the online setting, in which we collect user data while learning the optimal cache eviction policy, and which does not require an offline dataset.
For example, a deployed chatbot must decide whether to serve each incoming query from the cache or invoke the LLM at a cost for a stream of users, while continuously estimating query distributions and serving costs. 
We first introduce the online learning setup, then present the proposed algorithm \AlgTwo~in \Cref{alg:CLCB-LLM-switch}, and finally provide theoretical guarantees and analysis.

\subsection{Online Learning Setup}
We consider a $T$-round sequential decision-making problem. At the beginning of each round $t \in [T]$, the system maintains a semantic cache $\cM_t$ of size $k$.

A user with query $q_t \sim \bp$ arrives. As described in \Cref{sec:SC_serving}, the agent chooses between two actions:
\begin{itemize}
    \item  Serve $q_t$ using the cached response of its nearest neighbor $s(q_t, \cM_t)$, incurring a mismatch cost $d(q_t, \cM_t)$ but receiving no feedback on the true serving cost $c(q_t)$;
    \item  Query the LLM to obtain $a(q_t)$, incurring a realized cost $C_t(q_t)$ and observing this cost.
\end{itemize}

At the end of round $t$, the agent decides whether to update the cache. If $\cM_{t} \ne \cM_{t-1}$, the agent must query the LLM to fill the responses for the updated queries in $\cM_{t} \setminus \cM_{t-1}$. 
This models a \emph{switching cost} incurred for updating the cache. 

\textbf{Learning Objective.} The objective is to minimize the cumulative regret relative to the best fixed $\alpha$-approximate cache in hindsight, while accounting for both serving and switching costs. Formally, letting $\cM_0=\emptyset$, the regret with switching cost is defined as:
\begin{align}\label{eq:reg_def}\textstyle
    \text{Reg}(T) &= \E \left[ \sum_{t \in [T]} \ell(\cM_t; \bp,\bc,d) +  \sum_{q \in \cM_{t} \setminus \cM_{t-1}} c(q) \right] \notag\\
    &- \sum_{t\in [T]} \alpha 
    \ell(\cM^*;\bp,\bc,d)
\end{align}
where the expectation is taken on the randomness of the query arrival, the serving cost, and the cache selection.

As in standard online learning, minimizing regret requires balancing \emph{exploration}—to acquire feedback and refine estimates of the unknown serving costs—and \emph{exploitation}—to utilize the best-known cache to minimize loss \cite{dai2025variance}. This exploration-exploitation tradeoff is especially challenging in our setting because the agent only observes the serving cost when querying the LLM for a fresh response. Such queries can be costly, making effective exploration both essential and expensive.


\subsection{Algorithm Design of \AlgTwo}
\begin{algorithm}[!t]
    \caption{\texttt{CLCB-SC-LS}: Combinatorial Lower Confidence Bound Algorithm for Semantic Caching with Low Switching}
    \label{alg:CLCB-LLM-switch}
    \begin{algorithmic}[1]
        \State \textbf{Input:} Queries $\cQ$, solver \texttt{Reverse\_Greedy}, embedding generator $e$, distance function $d$, rounds $T$, probability $\delta$.
        \State \textbf{Initialize:} For each query $q \in \cQ$, set counter $N_{p,0}(q)=0, N_{c,0}(q)=0$, cumulative serving cost $L_{c,0}(q)=0$, stage index $\tau_q=1$, and round index sets $\cT(q, \tau_q)=\emptyset$. And initialize $\tau_p=1$, $\cT(p, \tau_p)=\emptyset$. 
        \For{round $t=1, ..., T$}
        \State User $t$ arrives with query $q_t \sim \bp$.
        \If{ $\exists \, q \in [m]$ s.t. $|\cT(q,\tau_q)| \ge 1+\sqrt{\frac{T \cdot \sum_{\tau=1}^{\tau_q-1} |\cT(q,\tau)|}{m}}$}\label{line:switch_q_start}
        \State Set $\tau_q = \tau_q$ + 1, $\cT(q, \tau_q)=\emptyset$.\label{line:switch_q_end}
        \State Set \texttt{Switch} = True.
        \ElsIf{$|\cT(p, \tau_p)| \ge 1 + \sqrt{T \cdot \sum_{\tau=1}^{\tau_p - 1} |\cT(p, \tau_p)|}$}\label{line:switch_p_start}
        \State Set $\tau_p = \tau_p+1$, $\cT(p, \tau_p)=0$.
        \State Set \texttt{Switch} = True.\label{line:switch_end}
        \EndIf
        \If{\texttt{Switch} == True}
        \State Call \texttt{Switch Cache} and get $\cM_t$.
        \State Set \texttt{Switch} = False
        \Else
        \State Remain $\cM_t = \cM_{t-1}$.
        \EndIf
        \State Call \texttt{Serve and Update}
        \EndFor
	\end{algorithmic}
\end{algorithm}

In this section, we present \AlgTwo, a low-switching online learning algorithm for semantic caching
The pseudocode is provided in \Cref{alg:CLCB-LLM-switch}, with subroutines \texttt{Switch Cache} in \Cref{alg:switch} and \texttt{Serve and Update} in \Cref{alg:serve_n_update}.

The algorithm proceeds in three key phases: (i) determining when to switch the cache based on the number of observations, (ii) actively interacting with the environment to explore the unknown serving costs and arrival probabilities, and (iii) serving user queries while collecting feedback to refine parameter estimates over time.

\noindent\textbf{Key Differences from Offline Learning.}
Compared to the offline algorithm (\AlgOne), which optimizes a single cache using a static dataset, the online setting requires the agent to explore different caches over time to learn the unknown serving costs. However, changing the cache too frequently incurs switching overhead. Thus, \AlgTwo~must both: (i) ensure sufficient exploration to estimate serving costs accurately, and (ii) limit the number of switches to control switching costs.


\noindent\textbf{Stage-Based Cache Switching (Lines \ref{line:switch_q_start}-\ref{line:switch_end} of \cref{alg:CLCB-LLM-switch}).}
To control switching frequency while ensuring sufficient exploration, we introduce a stage-based switch mechanism for both the serving cost estimates and the arrival probabilities.

For each query $q \in \cQ$, we partition the time slots into stages, where a new stage begins once sufficient serving cost feedback has been collected. Let $\cT(q, \tau)$ denote the set of rounds in which query $q$ was submitted to the LLM to get a fresh response during stage $\tau$. Let $\tau_q(t)$ denote the current stage index of $q$ at round $t$. As shown in \cref{line:switch_q_start}, stage $\tau_q$ is incremented if:
$
|\cT(q, \tau_q)| \ge 1 + \sqrt{ \frac{T \cdot \sum_{\tau=1}^{\tau_q-1} |\cT(q,\tau)| }{m} },
$
i.e., when the number of observations in the current stage exceeds a threshold (whose value is tuned to bound the number of switches as in \cref{lem:online_n_switch}) based on the cumulative prior observations.

Similarly, we maintain a global stage index $\tau_p$ for arrival probability estimation. This stage is incremented once the number of user arrivals in the current stage exceeds a corresponding threshold. As shown in \Cref{line:switch_p_start}, stage $\tau_p$ is updated when:
$
|\cT(p, \tau_p)| \ge 1 + \sqrt{T \cdot \sum_{\tau=1}^{\tau_p - 1} |\cT(p, \tau_p)|}.
$

Whenever a new stage is triggered for query $q$ or the arrival probability, the flag $\texttt{Switch}$ is changed to True and signals a new cache update in \cref{alg:switch}.

\noindent\textbf{Active Exploration via Optimism Principle (\cref{alg:switch}).}
To efficiently explore unknown serving costs, we follow the principle of optimism in the face of uncertainty \cite{dai2025unified,liu2023contextual}. Specifically,
we compute \textit{lower confidence bounds} (LCBs) as optimistic estimates in \cref{line:online_LLM_LCB} of \cref{alg:switch}:
$\ubar{c}_t(q) = \hat{c}_t(q) - \sqrt{ \frac{2 \log( \frac{4mn}{\delta} ) }{N_{c,t}(q)} }.$
The intuition is that for queries with limited feedback, subtracting the confidence radius from the empirical mean creates an optimistic estimate that encourages the learner to favor these uncertain queries, thus collecting more feedback on their unknown costs.

Once the \texttt{Switch} flag is true, we invoke \texttt{Switch Cache}, which uses the \texttt{Reverse\_Greedy} oracle with the estimated probabilities $\hat{\bp}_t$ and optimistic costs $\ubar{\bc}_t$ to construct a new cache $\cM_t$. The agent then calls the LLM to populate all new queries in $\cM_t \setminus \cM_{t-1}$. After updating the cache, the agent proceeds to serve queries and refine estimates using \texttt{Serve and Update} in the following.

\noindent\textbf{Serving and Updating (\Cref{alg:serve_n_update}).}
After determining the cache $\cM_t$ for the current round, the agent compares the lower confidence bound (LCB) estimate $\ubar{c}_t(q_t)$ with the mismatch cost incurred by using the nearest cached query $s(q_t, \cM_t)$ (see \Cref{line:online_compare}).
The \texttt{LLM} is queried only when $\ubar{c}_t(q_t) < d(q_t, \cM_t)$, as shown in \Cref{line:online_feedback} of \Cref{alg:serve_n_update}. The use of optimistic LCBs—where $\ubar{c}_t(q_t) \le c(q_t)$ with high probability—promotes exploration by increasing the likelihood of collecting real cost feedback. This mechanism is crucial for enabling the agent to gather informative data in the early rounds and for accelerating the learning of the unknown parameters.

If $\ubar{c}_t(q_t)$ is smaller than $d(q_t, \cM_t)$, the agent queries the \texttt{LLM}, incurs the realized serving cost $C_t(q_t)$, and uses the feedback to update its estimate of the expected cost $c(q_t)$. Otherwise, it returns the cached response $a(s(q_t, \cM_t))$ without receiving any feedback. Finally, the agent updates the arrival probability statistics based on the observed query $q_t$ (see \Cref{line:online_update_p}).
\begin{algorithm}[!t]
    \caption{\texttt{Switch Cache}}
    \label{alg:switch}
    \begin{algorithmic}[1]
        	\For{query $q \in \cQ$}
            \State Calculate empirical means $\hat{p}_t(q)=N_{p,t}(q)/t$ and $\,\hat{c}_t(q){=} L_{c,t}(q)/N_{c,t}(q)$.\label{line:online_LLM_empirical_mean}
            \State Calculate LCB of the cost $\ubar{c}_t(q)=\hat{c}_t(q)-\sqrt{ \frac{\log ({4mT^3})} {2N_{c,t}(q)} }$.\label{line:online_LLM_LCB}
        \EndFor 
        \State Compute $\cM_t= \texttt{Reverse\_Greedy} \left(\cQ, \hat{\bp}_t, \ubar{\bc}_t, d, k \right).$\label{line:online_LLM_oracle}
        \State Call \texttt{LLM} to obtain responses $a(q)$ for all $q \in \cM_t \,\backslash\, \cM_{t-1}$.
        \State \textbf{Return:} $\cM_t$.
	\end{algorithmic}
\end{algorithm}

\begin{algorithm}[!t]
    \caption{\texttt{Serve and Update}}
    \label{alg:serve_n_update}
    \begin{algorithmic}[1]
        \State $N_{p,t+1}(q)=N_{p,t}(q), N_{c,t+1}(q)=N_{c,t}(q), L_{c,t+1}(q)=L_{c,t}(q)$ for all $q \in \cQ$.
        \If{$\ubar{c}_t(q_t)<d(q_t, \cM_t)$}\label{line:online_compare}
        \State Call $a_t=\texttt{LLM}(q_t)$ with cost $C_t(q_t)\sim c(q_t)$, and serve the user with response $a_t$.\label{line:online_feedback}
        \State Update $N_{c,t+1}(q_t)=N_{c,t}(q_t) + 1$ and $L_{t+1}(q_t)=L_{t}(q_t) + C_t(q_t)$.
        \State Update $\cT(i, \tau_q) = \cT(i, \tau_q) \cup \{t\}$.
        \Else
        \State Serve the user with $a_t = s(q_t, \cM_t)$.
        \EndIf
        \State Set $N_{p,t+1}(q_t)=N_{p,t}(q_t) + 1$, $\cT(p, \tau_p) = \cT(p, \tau_p) \cup \{t\}$.\label{line:online_update_p}
	\end{algorithmic}
\end{algorithm}

\subsection{Theoretical Result and Its Discussion}
In this section, we present the main theoretical result characterizing the performance of \AlgTwo.

\begin{theorem}[Regret of \AlgTwo]\label{thm:online}
For the online semantic caching problem, the regret of \AlgTwo~is upper bounded as:
$\text{Reg}(T) \le O\left( \sqrt{m T\log ( mT  ) }\log\log T\right).$
\end{theorem}

\begin{remark}[Discussion of \Cref{thm:online}]
\Cref{thm:online} shows that the regret of \AlgTwo~grows sublinearly with the number of rounds $T$. Notably, the number of cache switches is bounded by $O(\log \log T)$, as established in \Cref{lem:online_n_switch}.
In the special case of exact matching (i.e., where semantic distance reduces to exact query identity), our result improves upon the regret bound in Zhu et al. \cite{zhu2023optimal} by a factor of $O(k\sqrt{m}/C_1)$, where $C_1 = \min_{q \in \cQ} c(q)$ denotes the minimum query cost. Furthermore, our bound matches the state-of-the-art result in Liu et al. \cite{liu2025offline} up to an $O(\log \log T)$ multiplicative factor, thanks to our stage-based cache update design. In contrast, both of these prior algorithms incur linear $O(T)$ switching costs.
\end{remark}

\subsection{Regret Analysis}
To prove \Cref{thm:online}, we establish a sequence of lemmas.
We first show that the empirical estimates of the arrival probabilities and serving costs concentrate around their true values uniformly over time. Compared to the offline setting, this requires an \emph{any-time} guarantee, introducing an additional $T$ term within the logarithmic terms due to union bounding over $t \in [T]$.

\begin{lemma}[Any-time concentration of cost and probability estimates]\label{lem:online_concen}
    For any $\delta>0$, we define the concentration events as follows:
$\cE_{\text{arvl}} \defeq \Big\{ \sum_{q \in \cQ}|\hat{p}_t(q) - p(q)| \le \sqrt{\frac{2m \log (2T/\delta)}{t}} \text{ for } t \in [T] \Big\}, \quad
\cE_{\text{cost}} \defeq \Big\{ |\hat{c}_t(q) - c_t(q)| \le \sqrt{\frac{\log (2mT^2/\delta)}{2N_{c,t}(q)}} \text{ for } q \in \cQ, t \in [T] \Big\}.$
Then both events hold with high probability: $\Pr[\cE_{\text{arvl}}] \ge 1 - \delta$ and $\Pr[\cE_{\text{cost}}] \ge 1 - \delta$.
\end{lemma}

To handle the partial feedback inherent in the caching system—where no serving cost is observed when the cache is used—we require a stronger loss function sensitivity result than \cref{lem:offline_smooth}, specifically tailored for optimistic cost estimates based on lower confidence bounds.

\begin{lemma}[Sensitivity of Loss under Optimistic Cost Estimates.]\label{lem:online_smooth}
For any cost parameters $ \ubar{\bc}, \bc \in [0,1]^m$ such that $\ubar{c}(q) \le c(q)$ for all $q \in \cQ$, the loss function satisfies that for any cache $\cM \subseteq \cQ$,
    $\ell(\cM; \bc, \bp, d) - \ell(\cM; \ubar{\bc}, \bp, d) \le \sum_{q\in \cQ:\, \ubar{c}(q) \le d(q, \cM)} p(q) (c(q) - \ubar{c}(q))$.
\end{lemma}

We also upper bound the number of cache switches for each query $q \in \cQ$ and the arrival probability.
\begin{lemma}[Number of cache switches]\label{lem:online_n_switch}
    We can upper bound the number of stages $\tau_q(T)$ and $\tau_p(T)$ by:
        $\E\left[ \sum_{q \in \cQ} \tau_q(T) \right] \le O(m \log \log T),
        \E[\tau_p(T)] \le O(\log \log T). $
\end{lemma}

\begin{proof}[Proof of \cref{thm:online}]
For the ease of exposition, we define auxiliary variables $N'_{c,t}(q), \hat{\bc}_t', \ubar{\bc}_t'$ and $N'_{t,p}, \hat{\bp}_t'$ as follows: If the selected cache at time $t$ is just switched in \cref{alg:CLCB-LLM-switch} (i.e., \texttt{Switch} == True), then $N'_{c,t}(q)=N_{c,t}(q), \hat{\bc}_t'=\hat{\bc}_t, \ubar{\bc}_t'=\ubar{\bc}_t$ and $N'_{t,p}=t,  \hat{\bp}_t'=\hat{\bp}_t$; Otherwise, if the previous cache is maintained (i.e.,  \texttt{Switch} == False), then $N'_{c,t}(q)=N'_{c,t-1}(q), \hat{\bc}_t'=\hat{\bc}_{t-1}'$ and $N'_{t,p}=N'_{t-1,p}, \hat{\bp}_t'=\hat{\bp}_{t-1}'$.

Under the high-probability events $\cE_{\text{arvl}}$ and $\cE_{\text{cost}}$, we decompose the regret (without switching cost) into three terms:
$\ell(\cM_t;\bc, \bp) - \alpha \ell(\cM^*;\bc, \bp) 
= \underbrace{\ell(\cM_t; \bc, \bp) - \ell(\cM_t; \ubar{\bc}'_t, \hat{\bp}'_t)}_{\text{estimation gap}} 
+ \underbrace{\ell(\cM_t; \ubar{\bc}'_t, \hat{\bp}'_t) - \alpha \ell(\cM^*;\ubar{\bc}'_t, \hat{\bp}'_t)}_{\text{greedy gap}} 
\\+ \underbrace{\alpha \ell(\cM^*;\ubar{\bc}'_t, \hat{\bp}'_t) - \alpha \ell(\cM^*;\bc, \bp)}_{\text{optimistic gap}} 
\le \ell(\cM_t; \bc, \bp) - \ell(\cM_t; \ubar{\bc}'_t, \hat{\bp}'_t) + \alpha \ell(\cM^*;\bc, \hat{\bp}'_t) - \alpha \ell(\cM^*;\bc, \bp) 
\le \sum_{q \in \cQ:\, \ubar{c}'_t(q) \le d(q, \cM)} p(q) (c(q) - \ubar{c}'_t(q)) + 2\norm{\hat{\bp}'_t - \bp}_1 
\le \sum_{q \in \cQ:\, \ubar{c}'_t(q) \le d(q, \cM)} 2p(q) \sqrt{\frac{\log(\frac{2mT^2}{\delta})}{2 N'_{c,t}(q)}} + 2\sqrt{\frac{2m \log(\frac{2T}{\delta})}{N'_{t,p}}}.$
where the first equality is due to adding and subtracting $\ell \left(\cM_t; \ubar{\bc}'_t, \hat{\bp}'_t, d \right)$ and $ \alpha  \ell \left(\cM^*;\ubar{\bc}'_t, \hat{\bp}'_t, d \right)$, the first inequality is due to the greedy gap $\le 0$ by the approximation ratio guarantee of the $\texttt{Reverse Greedy}$, $\ubar{\bc}_t' \le \bc$ by \cref{lem:offline_concen}, and $\ell(\cM;\bp,\bc)$ is monotone regarding $\bc$, the second inequality is due to applying \cref{lem:online_smooth} and \cref{lem:offline_smooth}, the last inequality is due to \cref{lem:online_concen}.

Now let us denote $\gamma\defeq \log(2mT^2/\delta)$ for the first term, we have:
$\E[ \sum_{t=1}^T \sum_{q\in \cQ:\, \ubar{c}'_t(q) \le d(q, \cM_t)} p(q)\sqrt{ 2\gamma / N'_{c,t}(q) } ] \\
= \E[ \sum_{t=1}^T \sum_{q = \tilde{q}_t} \sqrt{ 2\gamma / N'_{c,t}(q) } ] \\
= \E[ \sum_{i \in \cQ} \sum_{\tau \in \tau_q(T)} \sum_{t \in \cT(q, \tau)} \sqrt{ 2\gamma / N'_{c,t}(q) } ] \\
\le \E[ \sum_{i \in \cQ} \sum_{\tau \in \tau_q(T)} \sum_{t \in \cT(q, \tau)} \sqrt{ 2\gamma / \sum_{\tau'=1}^{\tau-1}|\cT(q,\tau')| } ] \\
\le \E[ \sum_{i \in \cQ} \sum_{\tau \in \tau_q(T)} |\cT(q, \tau)| \sqrt{ 2\gamma / \sum_{\tau'=1}^{\tau-1}|\cT(q,\tau')| } ] \\
\le \E[ \sum_{i \in \cQ} \sum_{\tau \in \tau_q(T)} 2\sqrt{T \cdot \sum_{\tau'=1}^{\tau_q-1} |\cT(q,\tau')| / m}\\\cdot \sqrt{ 2\gamma / \sum_{\tau'=1}^{\tau-1}|\cT(q,\tau')| } ] 
= \E[\sum_{q \in \cQ} \tau_q(T)] \cdot 2\sqrt{ 2T\gamma / m } \\
\le O(\sqrt{mT \gamma} \cdot \log \log T)$, 
where the first equality is due to the fact that any query $q$ is observed if and only if $\ubar{c}_t'(q) \le d(q, \cM_t)$ and by the triggering probability equivalence trick of \cite{liu2023contextual} in expectation $\sum_{\ubar{c}_t'(q) \le d(q, \cM_t)}p(q) \cdot x_q$ equals to $\E[\sum_{q = \tilde{q}_t}x_q]$ where $\tilde{q}_t=q_t$ if $q_t$ is the observed or $\tilde{q}_t=\emptyset$ otherwise, the second equality is due by exchanging the order of summation, the first inequlaity is due to $N'_{c,t}(q) \ge \sum_{\tau'=1}^{\tau-1}|\cT(q, \tau')|$, the last inequality is due to \cref{line:switch_q_start} of \cref{alg:CLCB-LLM-switch}.

For the second term, we can similarly derive that
$\E\left[ \sum_{t=1}^T 2\sqrt{2m \log(2T / \delta) / N'_{t,p}} \right] \le O\left( \sqrt{m T\log(2T / \delta)} \cdot \log \log T \right)$.
Finally, we set $\delta = 1/T$ and note that the number of cache switches is bounded by $O((m+1)\log \log T)$ by \cref{lem:online_n_switch}, yielding a lower-order regret term
$\E\left[ \sum_{q \in \cM_t \setminus \cM_{t-1}} c(q) \right] \le O(m^2 \log \log T)$ in \cref{eq:reg_def}, completing the proof.
\end{proof}
\section{Experiments}
In this section, we present the experimental results for all three of our settings, i.e., Algorithms \ref{alg:rg}, \ref{alg:CUCB-SC}, and \ref{alg:CLCB-LLM-switch}. 

\textbf{Experimental Setup.} 
We test these algorithms' performance in a synthetic caching setting where we generate the \textbf{queries} ($m=20$ is used for all experiments except for the ablation study on $m$, and ChatGPT was used to generate the queries), and there is a differing variety of similarity between queries in terms of their content. Some queries are quite similar to each other, while others are more distantly related. For example "Rome attractions", 
"top 10 tourist attractions and hidden gems in Rome", 
"AI papers", "recent breakthroughs in AI and machine learning research" are some of the queries in our dataset. Queries are sampled using the uniform distribution. To construct the \textbf{expected serving cost} $c(q)$ associated with each query $q$, we tokenize these queries and define the cost as the number of tokens associated with each query that is min-max normalized after adding $\epsilon_t(q)$ zero-mean Gaussian noise with a standard deviation of 0.05. 
To form the \textbf{embeddings} $e(q)$ associated with each query, we use a sentence transformer model \cite{reimers-2019-sentence-bert} that maps queries to a 384-dimensional vector representation. The similarity metric $d$ is the Euclidean distance between query embeddings. Similarly, the similarity metric is min-max normalized to ensure that it is in the range of 0-1. 
The shaded regions in Figure \ref{fig:combined} indicate the standard deviations as 10 runs were made with each run having a different random seed, which affects the query sampling for the incoming query stream and the noise associated with the cost of the queries.

\begin{figure*}[!t]
  \centering
  \includegraphics[width=0.98\textwidth]{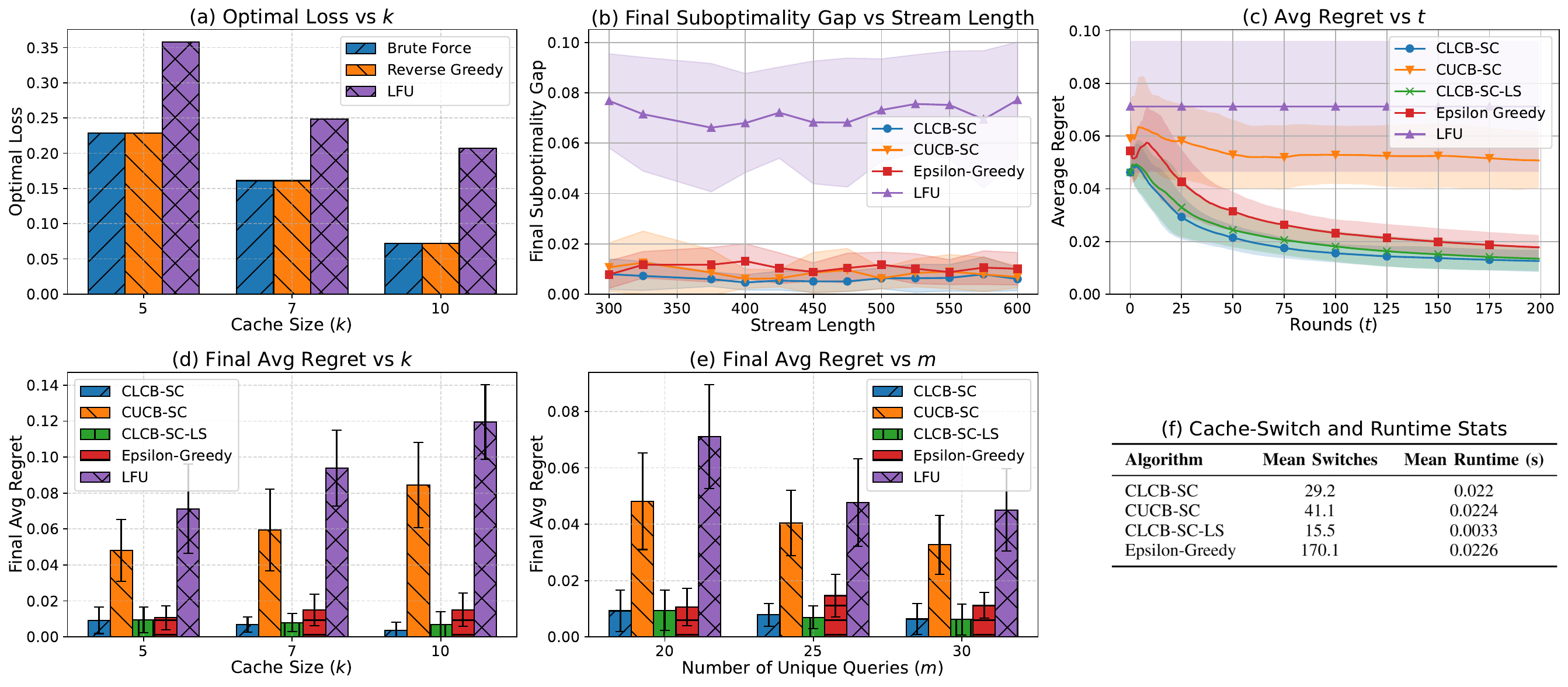}
  \caption{(a) Variation of the loss 
  with the cache size $k$ for a fixed set of queries with $m=20$ (b) Variation of the final suboptimality gap with the stream length in the offline setting, we test the performance of Algorithm \ref{alg:CUCB-SC} and its offline variant CLCB-SC (c) Variation of the average regret with the number of rounds $t$ for the online setting. CUCB-SC and CLCB-SC are the online variants of Algorithm \ref{alg:CUCB-SC} and its LCB variant (d) Ablation study on the final average regret's change with the cache size $k$ in the online setting (e) Ablation study on the final average regret's change with the number of distinct queries $m$ in the online setting (f) Table showing the mean number of switches and runtime of the algorithms in the online setting}
\label{fig:combined}
\end{figure*}

\textbf{Optimality of Reverse-Greedy with Known Parameters.} First we explore how Algorithm \ref{alg:rg} (Reverse Greedy) approximates the optimal loss $\ell(\cM^*; \bp, \bc, d)$ which we refer to here as the Brute Force optimal loss. We also plot the loss 
for the Least Frequently Used (LFU) caching strategy, which prioritizes caching the most frequent queries \cite{zhu2023optimal}.
 We show the variation of the optimal loss with the cache size $k$ in Figure \ref{fig:combined}(a). It can be seen that Reverse Greedy perfectly approximates the Brute Force optimal loss in this setting.


\textbf{Offline Setting.} Now we test how Algorithm \ref{alg:CUCB-SC} performs in the offline setting as seen in Figure \ref{fig:combined}(b). For these experiments we keep $m=20$ and $k=5$ and vary the length of the incoming query stream. We also consider the LCB variant of Algorithm \ref{alg:CUCB-SC} and Epsilon Greedy and LFU baselines. Epsilon Greedy, a standard benchmark for bandit algorithms, uses Reverse Greedy to construct the cache but uses the empirical cost $\hat{c}$ and randomly removes a query with probability $\epsilon_g$ (which we set to be 0.2 for all experiments).
CUCB-SC and CLCB-SC perform comparably in terms of the final suboptimality gap while Epsilon Greedy does slightly worse. This aligns with expectations: CUCB is theoretically optimal, while CLCB, though it may optimistically estimate costs, performs well in typical (non-worst-case) scenarios.

\textbf{Online Setting.} Next, we conduct an experiment for the online caching setting where we test the performance of CLCB-SC-LS (Algorithm \ref{alg:CLCB-LLM-switch}) and the online variant of CUCB-SC (Algorithm \ref{alg:CUCB-SC}) and its LCB variant (CLCB-SC, without adding the switching cost in regret) by recording the variation of the average regret with the number of rounds $t$ as shown in Figure \ref{fig:combined}(c). 
Finally, we consider LFU as a static baseline in this setting as we let it consider the full query stream before deciding on the cache.
It can be seen that CLCB-SC-LS and CLCB-SC have sublinear regret with the average regret approaching zero as round $t$ increases and achieves at least 11.75\% improvement compared with the most competitive Epsilon-Greedy baseline.


We now present ablation studies regarding the variation of the final average regret with $k$ and $m$. Figure \ref{fig:combined}(d) shows that CLCB-SC and CLCB-SC-LS get lower final regret as the cache size increases while CUCB-SC and LFU get higher final regret, showing our enlarged improvement compared with the strongest baseline (from 11.75\% to 54.04\%). 
Figure \ref{fig:combined}(e) shows the variation of the final average regret with $m$, notably CLCB-SC-LS and CLCB-SC achieve lower final average regret than the other baselines for as $m$ values, validating our tight regret bound in \cref{thm:online}.

Finally, Figure~\ref{fig:combined}(f) presents the total number of cache switches and running time in the online setting. As expected, the low-switching algorithm CLCB-SC-LS achieves the fewest switches and fastest runtime, reducing switches by up to 90.91\% and running time by up to 85.40\% compared to the baselines. Despite this constraint, its performance remains strong—Figure~\ref{fig:combined}(c) shows that its regret closely tracks that of CLCB-SC, which incurs nearly twice as many switches.




\section{Conclusion and Future Directions}
We present a principled, learning-based semantic caching framework for low-cost LLM serving. Our study spans three settings---oracle, offline, and online---offering a systematic study under different information regimes. The proposed algorithms achieve state-of-the-art performance guarantees and show superior empirical results. A promising future direction is to explore hybrid approaches that integrate offline training with online adaptation for more robust and efficient caching.
\section{Appendix}
Here we provide the missing proofs in the main body.

\begin{proof}[Proof of \cref{lem:NP}]
    For \cref{eq:bipartite_loss}, the optimization problem is equivalent to 
    $\cM^* = \argmax_{\cM \subseteq \cU, |\cM| \le k} \sum_{v \in N(\cM)} p(v)c(v),$
which is essentially the maximum vertex cover problem on bipartite graphs and is NP-hard as proved in \cite{10.1016/j.dam.2013.05.015}.
\end{proof}

\begin{proof}[Proof of \cref{lem:supermodular}]
    For each query $q \in \cQ$, define the per-query loss as $\ell_q(\cM) = \min\{ d(q, \cM),\, c(q) \} = \min_{u \in \cM \cup \{q_0\}} d(q, u)$, where $q_0$ is a virtual element with $d(q, q_0) = c(q)$. 
As submodularity and monotonicity are preserved under non-negative linear combinations.
It suffices to show that $\ell_q(\cM)$ is a \emph{monotone decreasing submodular function}. Let $A \subseteq B \subseteq \cQ$ and $u \notin B$.
For \textbf{Monotonicity}, since $\ell_q(\cM)$ takes the minimum over a larger set when moving from $A$ to $B$, it holds that $\ell_q(B) \le \ell_q(A)$. For \textbf{Submodularity}, we need to show the diminishing returns property:
$\ell_q(A \cup \{u\}) - \ell_q(A) \le \ell_q(B \cup \{u\}) - \ell_q(B).$
We exhaustively consider three cases based on the relative values of $\ell_q(A)$, $\ell_q(B)$, and $d(q,u)$; in all, the inequality holds.
\end{proof}




\begin{proof}[Proof of \cref{thm:approx}]
    The approximation guarantee follows from invoking Corollary 4 of \cite{il2001approximation}, which provides bounds for minimizing a non-increasing supermodular function under a cardinality constraint using the reverse greedy algorithm.
\end{proof}

\begin{proof}[Proof of \cref{lem:offline_concen}]
    For $\cE_{arvl}$, $\Pr[\cE_{arvl}] \ge 1-\delta $ holds because of Weissman et al. \cite{weissman2003inequalities}.
For $\cE_{cost}$ and $\cE_{arvl}$, we follow Lemma 5 and Lemma 6 of Liu et al. \cite{liu2025offline}, respectively.
\end{proof}

\begin{proof}[Proof of \cref{lem:offline_smooth}]
For any cache $\cM$, any $\bp, \bp', \bc, \bc' \in [0,1]^m$, we can rearrange $\bar{p}=(\max\{p(q), p'(q)\})_{q \in \cQ}$, $\ubar{p}=(\min\{p(q), p'(q)\})_{q \in \cQ}$, $\bar{c}=(\max\{c(q), c'(q)\})_{q \in \cQ}$, $\ubar{c}=(\min\{c(q), c'(q)\})_{q \in \cQ}$. Let us fix any $d$ and rewrite $\ell(\cM; \bp, \bc) \defeq \ell(\cM; \bp, \bc, d)$ for simplicity, then we have: $|\ell(\cM; \bp', \bc') - \ell(\cM; \bp, \bc')| \le \ell(\cM; \bar{\bp}, \bc') - \ell(\cM; \ubar{\bp}, \bc') = \sum_{q \in \cQ} (\bar{p}(q) - \ubar{p}(q)) \min\{c'(q), d(q, \cM)\} \le \sum_{q \in \cQ} (\bar{p}(q) - \ubar{p}(q)) = \sum_{q \in \cQ} |p(q) - p'(q)|$.
We can also prove:
$|\ell(\cM; \bp, \bc') - \ell(\cM; \bp, \bc)| \le \ell(\cM; \bp, \bar{\bc}) - \ell(\cM; \bp, \ubar{\bc}) = \sum_{q \in A} p(q)(\bar{c}(q)-\ubar{c}(q)) + \sum_{q \in B} p(q)(d(q,\cM)-\ubar{c}(q)) + \sum_{q \in C} p(q)(d(q,\cM)-d(q,\cM)) \le \sum_{q \in A \cup B} p(q)(\bar{c}(q)-\ubar{c}(q)) = \sum_{\ubar{c}(q) \le d(q,\cM)} p(q)(\bar{c}(q)-\ubar{c}(q)) \le \sum_{q \in \cQ} p(q)(\bar{c}(q)-\ubar{c}(q)) = \sum_{q \in \cQ} p(q)|c(q) - c'(q)|$,
where $A = \{q \in \cQ: d(q, \cM)\ge \bar{c}(q) > \ubar{c}(q) \}$, $B = \{q \in \cQ: \bar{c}(q) \ge d(q, \cM)>  \ubar{c}(q) \}$, and $C= \{q \in \cQ: \bar{c}(q) \ge  \ubar{c}(q) > d(q, \cM) \}$.
Finally, using the fact that $|\ell(\cM; \bp', \bc') - \ell(\cM; \bp, \bc)| \le |\ell(\cM; \bp', \bc') - \ell(\cM; \bp, \bc')| + |\ell(\cM; \bp, \bc') - \ell(\cM; \bp, \bc)|$ concludes the \cref{lem:offline_smooth}.
\end{proof}

\begin{proof}[Proof of \cref{lem:online_concen}]
    We follow the same proof of \cref{lem:offline_concen}, but use the union bound over all $t \in [T]$ and $N_{c,t}(q)\in [T]$. 
\end{proof}

\begin{proof}[Proof of \cref{lem:online_smooth}]
    We follow the proof of \cref{lem:offline_smooth} with $\bar{\bc}=\bc, \ubar{\bc}=\bc'$.
\end{proof}

\begin{proof}[Proof of \cref{lem:online_n_switch}]
Let us denote $s_0 = \log \log (T/m)$.
By the switch condition in \cref{line:switch_q_start} of \cref{alg:switch} and Lemma 2 in Dong et al. \cite{dong2020multinomial}, we have:
$|\cT(q, \tau_q)| \ge (T/m)^{1-2^{-\tau_q + 1}} \ge (T/m)^{1-2^{\log \log (T/m)}} = \frac{T}{e \cdot m}$
for any $\tau_q \ge s_0 + 1$ and $\tau_q$ is not the last stage.
Since the total number of observations is at most $T$ (one query each round), there are at most $m(e+1)$ (including their last stages for $q \in \cQ$) query-stage pairs $(q, \tau)$ s.t. for $q \in \cQ$, $\tau_q \in [\tau_q(T)]$ s.t. $\tau_q \ge s_0+1$ (which implies $\tau_q(T)\ge s_0+1$). Thus, we have $\sum_{q \in \cQ} \I\{\tau_q(T) \ge s_0 + 1\}\tau_q(T) = \sum_{q: \tau_q(T) \ge s_0+1}\sum_{\tau_q \in [\tau_q(T)]} \I\{\tau_q < s_0+1\} + \I\{\tau_q \ge s_0+1\}   < m (s_0+1) +  m(e+1)$. Therefore, we have $\E [\sum_{q \in \cQ} \tau_q(T)] = \E[\sum_{q \in \cQ} \I\{\tau_q(T) < s_0 + 1\}\tau_q(T) + \I\{\tau_q(T) \ge s_0 + 1\}\tau_q(T)] \le m (s_0+1) + [m(s_0+1) + m(e+1)]= m(2s_0 + e + 3)=O(m \log \log (T))$. Similarly, we can follow the same proof to show that $\E[\tau_p]\le O(\log \log (T))$ by treating the arrival probability as a single arm with $m=1$ that has total $T$ observations.
\end{proof}

\bibliographystyle{IEEEtran}
\bibliography{main}

\end{document}